\documentclass[10pt,journal,compsoc]{IEEEtran}

\ifCLASSOPTIONcompsoc
  \usepackage[nocompress]{cite}
\else
  \usepackage{cite}
\fi
\ifCLASSINFOpdf
\else
\fi

\usepackage{enumitem}
\usepackage{times}
\usepackage{epsfig}
\usepackage{graphicx}
\usepackage{amsmath}
\usepackage{amssymb}
\usepackage{amsthm}

\usepackage[pagebackref=true,breaklinks=true,colorlinks,bookmarks=false]{hyperref}
\newenvironment{Proposition}[1]{\par\noindent\textbf{\underline{Proposition}:}\space#1}{}
\newenvironment{Claim}[1]{\par\noindent\textbf{\underline{Claim}:}\space#1}{}

\usepackage{algorithm}
\usepackage{algpseudocode}
\usepackage{subcaption}

\let\bm\boldsymbol
\DeclareMathOperator*{\argmin}{arg\,min}
\DeclareMathOperator*{\argmax}{arg\,max}
\newtheorem{theorem}{Theorem}

\begin{document}
%
\title{Deep Dictionary Learning: A PARametric NETwork Approach}
%
%
%
%
\author{Shahin~Mahdizadehaghdam,~\IEEEmembership{Member,~IEEE,}
        Ashkan~Panahi,~\IEEEmembership{Member,~IEEE,}
        Hamid~Krim,~\IEEEmembership{Fellow,~IEEE,}
        Liyi~Dai,~\IEEEmembership{Fellow,~IEEE}
\IEEEcompsocitemizethanks{\IEEEcompsocthanksitem S. Mahdizadehaghdam, A. Panahi, and H. Krim are with the Department of Electrical and Computer Engineering, North Carolina State University, Raleigh, NC 27695 USA.\protect\\
E-mail: smahdiz@ncsu.edu; apanahi@ncsu.edu; ahk@ncsu.edu
\IEEEcompsocthanksitem L. Dai is with Army Research Office, RTP, Raleigh, NC 27703 USA.\protect\\
E-mail: liyi.dai.civ@mail.mil
}
}
\IEEEtitleabstractindextext{%
\begin{abstract}
\label{sec:abstract}
Deep dictionary learning seeks multiple dictionaries at different image scales to capture complementary coherent characteristics. We propose a method for learning a hierarchy of synthesis dictionaries with an image classification goal. 
The dictionaries and classification parameters are trained by a classification objective, and the sparse features are extracted by reducing a reconstruction loss in each layer.
The reconstruction objectives in some sense regularize the classification problem and inject source signal information in the extracted features.
The performance of the proposed hierarchical method increases by adding more layers, which consequently makes this model easier to tune and adapt.
The proposed algorithm furthermore, shows remarkably lower fooling rate in presence of adversarial perturbation.
The validation of the proposed approach is based on its classification performance using four benchmark datasets and is compared to a CNN of similar size.
\end{abstract}

}

\maketitle

\IEEEdisplaynontitleabstractindextext

%
\IEEEpeerreviewmaketitle

\ifCLASSOPTIONcompsoc
\IEEEraisesectionheading{\section{Introduction}\label{sec:introduction}}
\else
\section{Introduction}
\label{sec:introduction}
\fi

%
%
%
%



\IEEEPARstart{T}{he} key step to the complex task of classifying images is that of obtaining features of these images which encompass relevant information, e.g., label information. The two most well-known research directions in this regard are \textit{Deep Neural Networks} and \textit{Dictionary Learning for Sparse Representation}.

\textit{Deep Neural Networks:} In recent years, Deep Neural Networks (DNNs), and more specifically Convolutional Neural Networks (CNNs) \cite{Cun90handwrittendigit, lecun1998gradient}, showed impressive results in many applications, in particular, signal and image processing \cite{dong2016image, ulyanov2016texture}. 
A Convolutional Neural Network consists of multiple layers and a different number of filters in each layer. 
Despite these significant achievements, there is still little theoretical understanding of the learning process in these networks.
Invariant scattering convolution \cite{bruna2013invariant} is among the few works to provide a theoretical perspective of CNN.
This technique specializes the filters in CNN to be fixed wavelet functions. As the wavelet transform is invariant to translation and rotation, the features from the scattering transform are invariant to these transformations as well. 

\textit{Dictionary Learning for Sparse Representation:}
Parsimonious data representation by learning overcomplete dictionaries has shown promising results in a variety of problems such as image denoising \cite{elad2006image}, image restoration \cite{xu2016cloud}, audio processing \cite{grosse2012shift}, and image classification \cite{Zhang2016168}. 
This frame-like representation of each data vector as a linear combination of atoms carries a sparse notion of the associated coefficients. 
Using Sparse Representation-based Classification (SRC) \cite{wright2009robust}, one can represent an image as a combination of a few images in the training dataset.
This is followed subsequently by a classifier based on these feature vectors.
The proposed refinements were task-driven dictionary learning \cite{mairal2012task} and Label Consistent K-SVD (LC-KSVD) \cite{jiang2013label} which jointly learn an overcomplete dictionary, sparse representation, and classification parameters. 

The aforementioned dictionary learning methods are based on entire images for training the dictionary and finding the sparse representation, which can be computationally expensive. 
These potentially lead to a poor performance when the training dataset is small.
Convolutional Neural Networks, however, learn the initial features from small image patches and build a hierarchy of the features at different scales. 
Contrary to conventional wisdom, several experimental studies \cite{he2016deep, he2015convolutional} have reported that deeper neural networks are more difficult to train, and adding more layers, eventually leads to decreased performance. 
Part of the this mentioned problem is due to the vanishing/exploding gradient effect during training.
This problem persists despite mitigations such as batch normalization \cite{he2016deep}.  
To cope with the CNN's fore-noted limitations, and to exploit the deep structure intrinsic to data, we propose a principled hierarchical (deep) dictionary learning to be learned while achieving optimal classification.
Within this framework, the front layer dictionary is learned on small image patches, and the subsequent layer dictionaries are learned on larger scales.
Put simply, the initial scale captures the fine low-level structures comprising the image vectors, while the next scales coherently capture more complex structures.
The classification is ultimately carried out by assembling the final and largest scale features of an image and assessing their contribution. 
In contrast to CNN, we show that the performance of the proposed DDL method improves with additional layers, hence indicating an amenability to tuning, and a better potential for more elaborate learning tasks such as transfer learning.

Inspired by the study in \cite{shwartz2017opening} on neural networks, we show, using an information theoretic argument that DDL is a sensible approach to image classification.
We show that under a certain generative model, DDL maximizes the mutual information $I(\bm{A^*}, \bm{Y})$ between the optimal representation of the observed data/signals and the labels. The optimal representation, $\bm{A^*}$, is obtained by maximizing their mutual information $I(\bm{X}, \bm{A})$ with the input signals.
In the dictionary learning framework, we show that the proposed model simplifies to a joint learning of the dictionary and the classification parameter, which we subsequently generalize into hierarchically/deeply learning the dictionaries over different layers.

On account of the parsimony of our dictionary atom representation of features at each layer, our proposed method exhibits a remarkably lower fooling rate to adversarial perturbation and random noises. In light of the limited performance of the state of the art classifiers as a result of a single additive perturbation \cite{Dezfooli2017, deepfool, goodfellow2014explaining}, the robustness displayed by our proposed algorithm is significant and points to a promise of the approach. Our contributions in this paper are summarized as as follows:

\begin{itemize}[leftmargin=*]
\item{One of the novelties of this paper is in rigorously showing the importance of aiming for a high mutual information between the original signal and the output of each layer. 
This is in fact in agreement with the conclusion of the Residual Network \cite{he2016deep}, where a direct connection of increasing the number of layers in CNN with the drop in performance has been made. 
Our paper establishes that preserving a high fidelity to the input signal (by way of a layer-based LASSO-regularization) is the key to having deeper models with no performance loss in  classification.}
\item{The proposed algorithm is robust to adversarial perturbation and random noises.}
\item{A theoretical discussion on the model-hyperparameter selection is provided. We established a relation between the second moment of the input signal and the necessary width of each layer.}
\end{itemize}

The balance of the paper is organized as follows: 
In Section \ref{sec:prel}, we provide the problem statement as well as some background information of relevance to this paper.
We formulate and propose our new approach in Section \ref{sec:model}. 
The neural networks and the proposed method are compared from an information theoretic point of view in Section \ref{sec:MI}.
The theoretical discussion about the hyper-parameter selection is in Section \ref{sec:Width_of_Network}.
Substantiating experimental results are presented in Section \ref{sec:Results}. Finally, we provide some concluding remarks in Section \ref{sec:conclusions}.

\section{Problem Statement and Preliminaries}
\label{sec:prel}
Image classification is typically based on learning a synthesis dictionary which yields representations of each image as a sparse linear combination of the atoms of the learned dictionary \cite{Mairal:2009:ODL}. 
This is typically followed by a classification technique, such as SVM \cite{hearst1998support}, a neural network, or a linear classifier operating on the sparse feature vectors.
With that goal in mind, one can vectorize all training images into a matrix and perform dictionary learning and sparse representation \cite{mairal2009online}.

Given vectorized images, $\bm{g}_i \in R^m \; i\in \{1,.., n\}$, as a matrix $\bm{G}$, 
the optimal dictionary $\bm{D}^* \in R^{m\times k}$ and
the optimal sparse representations of the images, $\bm{a}^*_i \in R^k$ (the columns of the matrix $\bm{A}^*$), can be be learned by minimizing the reconstruction loss function $\mathcal{L^R}(\bm{D}, \bm{A}, \bm{G})$:
\begin{equation}
\begin{aligned}
&\;\;\;\;\;\;\;\;\;\;\;\;\{\bm{A}^*, \bm{D}^*\}= \argmin_{\substack{\bm{A},\;\bm{D} }}\; \mathcal{L^R}(\bm{D}, \bm{A}, \bm{G}),\\
&\mathcal{L^R}(\bm{D}, \bm{A}, \bm{G}) = \frac{1}{2}||\bm{G} - \bm{D}\bm{A}||_F^2  + \lambda ||\bm{A}||_1 + \lambda^{\prime} ||\bm{A}||^2_F,\\   \;\;\; &\bm{A} \geq 0, \;\; and \;\; \bm{D} \in \mathcal{C},\\
\end{aligned}
\label{eq:dict_learning}
\end{equation}
where $\mathcal{C}$ is the convex set of matrices with unit $L_2$-norm columns. The regularizers of Eqn. (\ref{eq:dict_learning}) are for a sample setting and may vary with the problem at hand.
Different types of regularization may be imposed on the feature vectors for specific task. 
Having the optimal representation $\bm{A}^*$, and the label information of the images $\bm{Y}$, the desired classifier can be trained by minimizing the classification loss function $\mathcal{L^C}(\bm{Y}, \bm{A}^*, \bm{W})$ over the classification parameter $\bm{W}$:
\begin{equation}
\bm{W}^*=\argmin_{\substack{\bm{W}}}\; \mathcal{L^C}(\bm{Y}, \bm{A}^*, \bm{W}).
\label{eq:classification}
\end{equation}

Although the formulation in Eqn. (\ref{eq:dict_learning}) can achieve a very low error in reconstruction of the original image from the extracted sparse features, these  feature vectors are not necessarily optimal for classification purposes.
More generally, studies in recent years have shown that isolating the dictionary learning from classification yields suboptimal dictionaries for classification purposes \cite{mahdizadehaghdam2017image, akhtar2016discriminative}.
Thus, more advanced methods have been proposed to jointly train the dictionary and the classification models. 
These methods, in general, attempt to learn a dictionary for the purpose of classification.
Among these, figure supervised dictionary learning methods \cite{mairal2009supervised}. 
Aside from the subtle differences among the latter methods, they are commonly based on jointly learning classification parameters, a dictionary and sparse representations from a loss function, which is usually a summation of a reconstruction loss (similar to $\mathcal{L^R}$ in Eqn. (\ref{eq:dict_learning})) and a classification loss:
\begin{equation}
\{\bm{W}^*, \bm{D}^*, \bm{A}^*\}=\argmin_{\substack{\bm{W}, \bm{D}, \bm{A}}}\; \mathcal{L^R}(\bm{D}, \bm{A}, \bm{G}) + \mathcal{L^C}(\bm{Y}, \bm{A}, \bm{W}).
\label{eq:joint1}
\end{equation}
Task-driven dictionary learning methods \cite{mairal2012task}, on the other hand, seek the optimum dictionary and classification parameters by minimizing a loss function which is based on only the classification loss. The feature vectors in these cases are, however, conditioned to form optimal sparse representation:
\begin{equation}
\begin{aligned}
\{\bm{W}^*, \bm{D}^*\}=&\argmin_{\substack{\bm{W}, \bm{D}}}\; \mathcal{L^C}(\bm{Y}, \bm{A}^*, \bm{W}),\\
s.t.: \;\;\;\;\;\;\;
\bm{A}^*=&\argmin_{\substack{\bm{A}}}\; \mathcal{L^R}(\bm{D}, \bm{A}, \bm{G}).
\end{aligned}
\label{eq:joint2}
\end{equation}
\begin{figure*}[t]
\centering
	\includegraphics[scale=0.37]
	{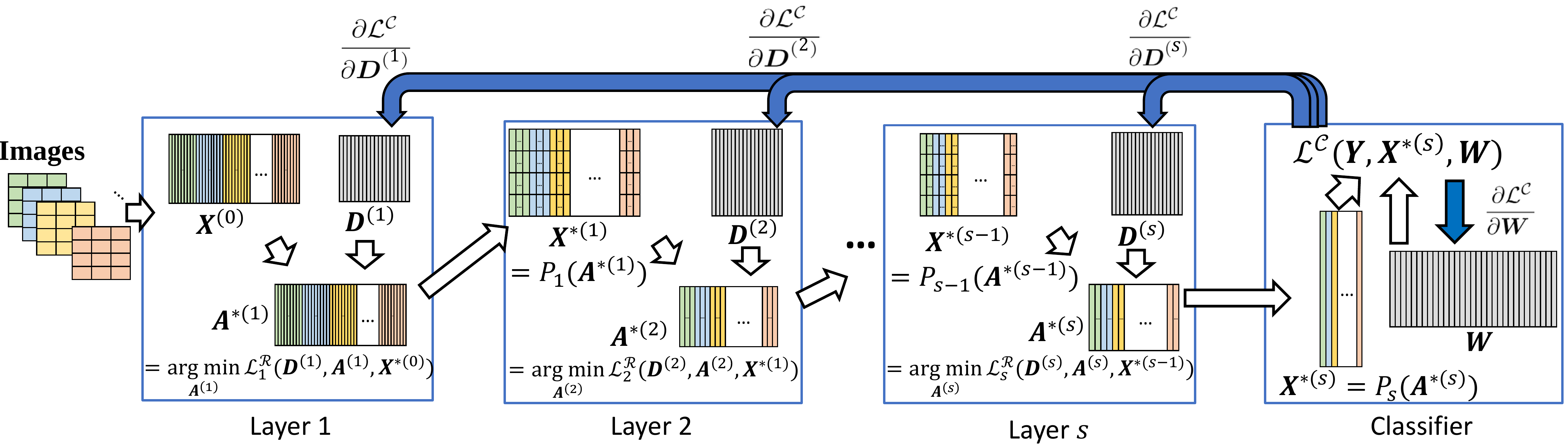}
	\caption{Sequential steps of a deep dictionary with $s$ layers.}
	\label{fig:matrices}
\end{figure*}
Despite the similarities to formulation Eqns. (\ref{eq:classification}) \& (\ref{eq:joint1}), the latter one is relatively more general, achieves a higher performance for large image datasets and is easier to train by a gradient decent approach \cite{mairal2012task}. 

Our work here, is in line with the task-driven dictionary learning methods with an additional set of deep dictionaries to capture structure at different scales \cite{mahdizadehaghdam2017image}.
\section{Proposed Solution}
\label{sec:model}
Our approach to classifying images starts by learning a classification model with input features/coefficients obtained from the last layer of a sequential hierarchy of dictionaries. Specifically, for an $s$-layer hierarchy, the classification parameters, and the dictionaries are learned by minimizing the following classification loss functional:
\begin{equation}
\begin{aligned}
\{\bm{W}^*, \{\bm{D}^{*(r)}\}_{r=1}^s\}&=\argmin_{\substack{\bm{W}, \;  \{\bm{D}^{(r)}\}_{r=1}^s}}\; \mathcal{L^C}(\bm{Y}, \bm{X}^{*(s)}, \bm{W}),
\end{aligned}
\label{eq:Hir_classification}
\end{equation}
where $\bm{W}$, $\bm{D}^{(r)}$, and $\bm{Y}$ are respectively the classification parameters, the dictionary for layer $r$, and the labels of the training images. 
$\bm{X}^{*(s)}$ is the input to the classifier which is calculated by concatenating the output vectors $\bm{A}^{*(s)}$ of the last layer as shortly explained. More generally, $\bm{X}^{*(r)}$ denotes the input to the $(r+1)^{th}$ layer while $\bm{A}^{*(r)}$ refers to the output of the $r^{th}$ layer. These vectors are the result of the following recursive relation:
\begin{equation}
\begin{aligned}
&\;\;\;\;\;\;\;\bm{A}^{*(r)}=\argmin_{\substack{\bm{A}^{(r)} }}\; \mathcal{L}_r^\mathcal{R} (\bm{D}^{(r)}, \bm{A}^{(r)}, \bm{X}^{*(r-1)}),\\
s.&t.:\\
&\mathcal{L}_r^\mathcal{R} (\bm{D}^{(r)}, \bm{A}^{(r)}, \bm{X}^{*(r-1)}) = \frac{1}{2}||\bm{X}^{*(r-1)} - \bm{D}^{(r)}\bm{A}^{(r)}||_F^2 \\
&\;\;\;\;\;\;\;\;\;\;\;\;\;\;\;\;\;\;\;\;\;\;\;\;\;\;\;\;\;\;\;\;\;\;\;\;+ \lambda ||\bm{A}^{(r)}||_1, + \lambda^{\prime} ||\bm{A}^{(r)}||^2_F,\\
&\bm{X}^{*(r)}=P_r (\bm{A}^{*(r)}),\\
&\bm{D}^{(r)} \in \mathcal{C}, \;\; \bm{A}^{(r)} \geq 0,
\label{eq:Hir_dict_learning}
\end{aligned}
\end{equation}
where each column of the matrix $\bm{X}^{*(0)}$ is a vectorized image patch, and $P_r$ is an operator which concatenates the feature vectors of adjacent patches from the previous layer. 
In other words, $P_r$ is an operator which reshapes the output of layer $r$ as the input to layer $(r+1)$.
Eqn. (\ref{eq:Hir_dict_learning}) is similar to an elastic net regularization problem with a non-negativity condition on the feature vectors.
We emphasize that the objective functional in Eqn. (\ref{eq:Hir_classification}) implicitly depends on all dictionaries, since the computation of  $\bm{X}^{*(s)}$ in Eqn. (\ref{eq:Hir_dict_learning}) requires them.

We show in Fig. \ref{fig:matrices} the sequence of computational steps exploiting matrices from Eqn. (\ref{eq:Hir_dict_learning}) and their associated structures. 
The white arrows in the figure depict the forward computing path of the sparse representation of each input layer. 
A set of images are first segmented into small image patches and vectorized into matrix $\bm{X}^{(0)}$ as the input of the first layer.
For a given first layer dictionary $\bm{D}^{(1)}$, the sparse representations of the image patches are learned as the columns of the $\bm{A}^{*(1)}$ matrix, by solving Eqn. (\ref{eq:Hir_dict_learning}) via the sparse encoding algorithm FISTA \cite{beck2009fast}.
The patches of $\bm{A}^{*(1)}$ ($3 \times 3$ window) are next reshaped by operator $P_1$ to yield input $\bm{X}^{*(1)}$ to the second layer. This analysis is sequentially carried on to scale $s$ where $\bm{A}^{*(s)}$ yields $\bm{X}^{*(s)}$ as an input to a classifier (such that each column of this matrix represents an image). The class label information of the images in matrix $\bm{Y}$, together with $\bm{X}^{*(s)}$ provide the classification parameters $\bm{W}$ as a solution to Eqn. (\ref{eq:Hir_classification}).


The blue arrows in Fig. \ref{fig:matrices} highlight the backward training path, reflecting the updates on the dictionaries as a result of the optimized classification, by way of gradient descent on Eqn. (\ref{eq:Hir_classification}).
Optimizing the classification loss $\mathcal{L}^C$ in the forward pass (path of Fig. \ref{fig:matrices}), is followed by updates on the backward pass (path of Fig. \ref{fig:matrices}). 

In summary, the relevant parameter vectors of the images are learned through multiple forward-backward passes through the layers. The dictionaries are updated such that the resulting representations are suitable for the classification. Moreover, the representations are learned by minimizing the reconstruction loss at each layer. 
The search for the optimal classifier is not only carried out by accounting for the deep structure of images, it is additionally regularized by preserving the fidelity of image content representation.
We further discuss the derivation rationale of our purposed approach in Section \ref{sec:MI}.



\subsection{Algorithm}
\label{subsec:alg}
Algorithm \ref{alg}, discusses the details of the updating procedure of the parameters in our method.

\textit{Lines 1-4:}
We randomly initialize the dictionaries and the classification parameter.

\textit{Lines 6-12:}
During the training phase, we randomly select a subset of images, and we construct the representation of the images by sequentially solving the problem in (line 9) and preparing the input for the next layer in (line 10). Our solution to the problem in (line 9) is obtained via the sparse coding algorithm FISTA \cite{beck2009fast}.
In (line 12), we calculate the classification loss for the selected subset of the images.

\textit{Lines 13-15:}
The classification parameter and the dictionaries are updated via the gradient of the computed loss in (lines 13 and 15), respectively. $\eta$ is the step size. 

\begin{algorithm}[!t]
\caption{}\label{alg}
\textbf{Initialization:}
\begin{algorithmic}[1]
\For{$r \text{ \textbf{in} } \{1,s\}$}:
    \State{Initialize $\bm{D}_0^{(r)}$ randomly.}
\EndFor
\State{Initialize $\bm{W}_0$ randomly.}
\algstore{myalg}
\end{algorithmic}
\textbf{Training:}
\begin{algorithmic}[1]
\algrestore{myalg}
\For{$t \text{ \textbf{in} } \{1,T\}$}:
        \Statex {\footnotesize{~}} 
        \Statex {~~~~~~\textbf{\textit{Forward pass}}:} 
        \State{\small{Randomly select a set of images from the training dataset.}}
        \State{\small{Patch and vectorize the images as the columns of $\bm{X}_t^{(0)}$.}}
        \For{$r \text{ \textbf{in} } \{1,s\}$}:
        
            \State{\small{$\bm{A}_t^{*(r)}=\argmin_{\substack{\bm{A}_t^{(r)} }}\; \mathcal{L}_r^\mathcal{R} (\bm{D}_t^{(r)}, \bm{A}_t^{(r)}, \bm{X}_t^{*(r-1)})$}}.
            
            \State{\small{$\bm{X}_t^{*(r)}=P_r (\bm{A}_t^{*(r)})$}}.
        \EndFor
        \State{\small{$\mathcal{L}_t^C=\mathcal{L^C}(\bm{Y}_t, \bm{X}_t^{*(s)}, \bm{W}_t)$}}.
        \Statex {\footnotesize{~}} 
        \Statex {~~~~~~\textbf{\textit{Backward pass}}:} 
        \State{$\bm{W}_{t+1}=\bm{W}_{t}-\eta \frac{\partial \mathcal{L}_t^C}{\partial \bm{W}_{t}}$}.
        \For{$r \text{ \textbf{in} } \{1,s\}$}:
            \State{$\bm{D}^{(r)}_{t+1}=\bm{D}^{(r)}_{t}-\eta \frac{\partial \mathcal{L}_t^C}{\partial \bm{D}^{(r)}_{t}}$}.
        \EndFor
      \EndFor
\end{algorithmic}
\end{algorithm}

\subsection{Optimization by Gradient Descent}
Depending on the choice of the classification functional, computing the gradient of the loss functional with respect to the classification parameter, $\frac{\partial \mathcal{L^C}}{\partial \bm{W}}$, can conceptually be straightforward. 
To ensure an updating step of the stochastic gradient descent for the dictionary at layer $r$, we require $\frac{\partial\mathcal{L^C}}{\partial \bm{D}^{(r)}}$.

\begin{Claim}
\textit{The matrix form of $\frac{\partial\mathcal{L^C}}{\partial \bm{D}^{(r)}}$ is given by,}

\begin{equation}
\begin{aligned}
\frac{\partial\mathcal{L^C}}{\partial \bm{D}^{(r)}} &= - \bm{D}^{(r)} \bm{\beta}^{(r)} \bm{a}^{*(r)T} + (\bar{\bm{x}}^{*(r-1)}-\bm{D}^{(r)}\bm{a}^{*(r)}) \bm{\beta}^{(r)T},\\
\bm{\beta}^{(r)}_{\bm{\Lambda}} &= (\bm{D}^{(r)T}_{\bm{\Lambda}} \bm{D}^{(r)}_{\bm{\Lambda}} + \lambda^{\prime} \bm{I}) ^{-1} \cdot \frac{\partial\mathcal{L^C}}{\partial \bm{a}^{*(r)}_{\bm{\Lambda}}},\\
\bm{\beta}_{\bm{\Lambda}^C}^{(r)}& =0,
\label{eq:opt_cond_vect_diff0}
\end{aligned}
\end{equation}
\textit{where $\bm{\Lambda}$ is the active set of $\bm{a}^*$,  $\bm{\Lambda} \triangleq \{j\;\; | \;\; \bm{a}^*[j] \neq 0, \; \; j \in \{1,..,k\}\}$, $\bm{\Lambda}^C$ is the complement of set $\bm{\Lambda}$, $\bm{I}$ is an identity matrix, $\bm{1}$ is a one-vector, and $\bar{\bm{x}}^{*(r)}$ is an arbitrary column of $\bm{X}^{*(r)}$.}
\end{Claim}

The aforementioned gradient derivation is built on the work in \cite{mairal2012task} for a deep network. The steps of deriving this gradient by the chain rule is as follows,
\begin{equation}
\begin{aligned}
\frac{\partial\mathcal{L^C}}{\partial \bm{D}^{(r)}} &=(\frac{\partial\mathcal{L^C}}{\partial \bar{\bm{x}}^{*(r)}})^T \cdot [\frac{\partial \bar{x}_1^{*(r)}}{\partial \bm{D}^{(r)}}, .., \frac{\partial \bar{x}_m^{*(r)}}{\partial \bm{D}^{(r)}}],\\
\frac{\partial\bar{\bm{x}}^{*(r)}}{\partial \bm{D}^{(r)}} &=
P_r(\frac{\partial\bar{\bm{A}}^{*(r)}}{\partial\bm{D}^{(r)}}),
\end{aligned}
\label{eq:chain_rule_vect}
\end{equation}
where $\bar{\bm{x}}^{*(r)}$ is an arbitrary column of $\bm{X}^{*(r)}$, and $\bar{\bm{A}}^{*(r)}$ are the corresponding output feature vectors from layer $r$ such that $P_r(\bar{\bm{A}}^{*(r)})=\bar{\bm{x}}^{*(r)}$. In addition, $\bar{x}_i^{*(r)}$, is the $i^{th}$ element of the vector $\bar{\bm{x}}^{*(r)}$.
The term $\frac{\partial\bar{\bm{A}}^{*(r)}}{\partial\bm{D}^{(r)}}$, is the gradient of a matrix with respect to a matrix, it is organized as a tensor with rank 4. 
$\frac{\partial\bar{\bm{x}}^{*(r)}}{\partial\bm{D}^{(r)}}$ is also a tensor of rank 3.
Because $\bar{\bm{x}}^{*(r)}$ is obtained as a result of a reshaping operator on $\bar{\bm{A}}^{*(r)}$, the same reshaping operator can be applied to compute the gradient.
The term $\frac{\partial\mathcal{L^C}}{\partial \bar{\bm{x}}^{*(r)}}$ can be computed by applying the chain rule.

In order to compute $\frac{\partial\bar{\bm{A}}^{*(r)}}{\partial\bm{D}^{(r)}}$, we temporarily drop the superscripts for simplicity.
Consider $\bm{d}_j$ as the $j^{th}$ column of the $\bm{D}^{(r)}$ matrix, $\bm{x}$ as a column of $\bm{X}^{*(r-1)}$ matrix, and $\bm{a}$ as the corresponding feature vector.
The vector $\bm{a}^{*}$ is an optimal solution of Eqn. (\ref{eq:Hir_dict_learning}) if and only if, 
\begin{equation}
(\bm{D}^T_{\bm{\Lambda}} \bm{D_\Lambda} + \lambda^{\prime} \bm{I}) \bm{a}^*_{\bm{\Lambda}} = \bm{D}^T_{\bm{\Lambda}} \bm{x} - \lambda \bm{1},
\label{eq:opt_cond_vect}
\end{equation}
We carry out element-wise computation of $\frac{\partial\bm{a}^{*}}{\partial\bm{D}}$,  by taking the gradient of both sides of Eqn. (\ref{eq:opt_cond_vect}) with respect to the elements of the $\bm{D}$ matrix:
\begin{equation}
\frac{\partial\bm{a}^*_{\bm{\Lambda}}}{\partial \bm{d}_{(i,j)}} = (\bm{D}^T_{\bm{\Lambda}} \bm{D_\Lambda} + \lambda^{\prime} \bm{I}) ^{-1} (\frac{\partial \bm{D}^T_{\bm{\Lambda}} \bm{x}}{\partial \bm{d}_{(i,j)}} - \frac{\partial \bm{D}^T_{\bm{\Lambda}} \bm{D}_{\bm{\Lambda}}}{\partial \bm{d}_{(i,j)}} \bm{a}^*_{\bm{\Lambda}}).
\label{eq:opt_cond_vect_diff}
\end{equation}

Further simplifying the equation above and restoring the superscripts, we can write the $\frac{\partial\mathcal{L^C}}{\partial \bm{D}^{(r)}}$ in matrix format as follows,
\begin{equation}
\begin{aligned}
\frac{\partial\mathcal{L^C}}{\partial \bm{D}^{(r)}} &= - \bm{D}^{(r)} \bm{\beta}^{(r)} \bm{a}^{*(r)T} + (\bar{\bm{x}}^{*(r-1)}-\bm{D}^{(r)}\bm{a}^{*(r)}) \bm{\beta}^{(r)T},\\
\bm{\beta}^{(r)}_{\bm{\Lambda}} &= (\bm{D}^{(r)T}_{\bm{\Lambda}} \bm{D}^{(r)}_{\bm{\Lambda}} + \lambda^{\prime} \bm{I}) ^{-1} \cdot \frac{\partial\mathcal{L^C}}{\partial \bm{a}^{*(r)}_{\bm{\Lambda}}},\\
\bm{\beta}_{\bm{\Lambda}^C}^{(r)}& =0.
\label{eq:opt_cond_vect_diff3}
\end{aligned}
\end{equation}
 
\section{An Information Theoretic Perspective of Deep Learning}
\label{sec:MI}
\begin{figure}
\centering
\includegraphics[scale=0.37]
{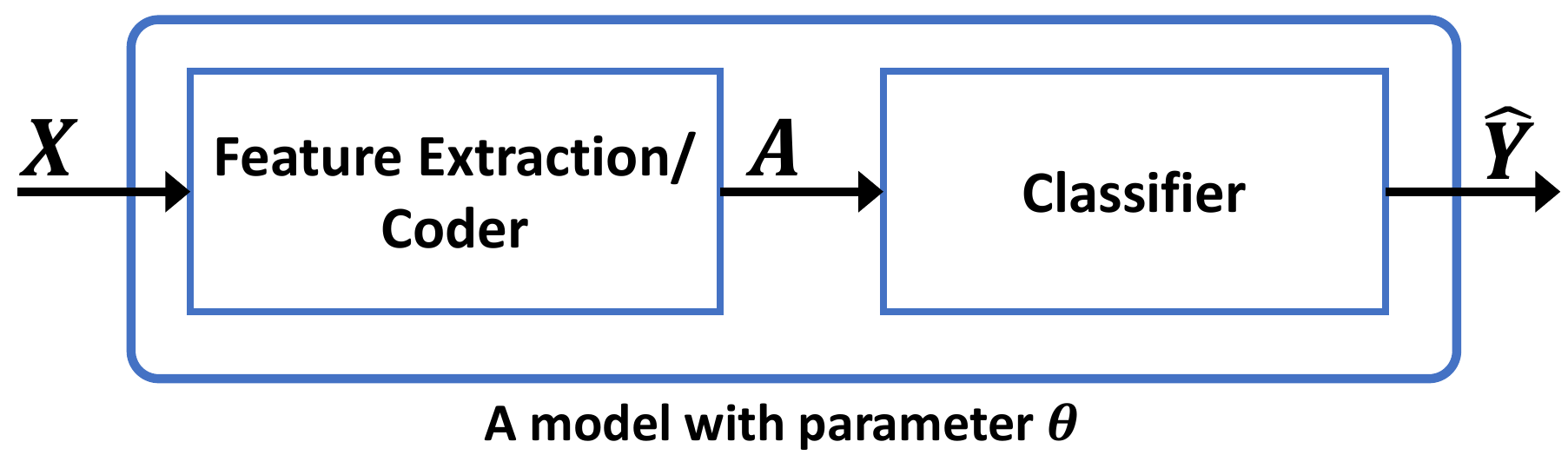}
\caption{$\bm{X}$ and $\bm{A}$ are respectively, the images and the associated feature vectors. $\bm{Y}$ represents the class labels, and $\hat{\bm{Y}}$ is the estimated class labels by the classifier.}
\label{fig:MI}
\end{figure}
The goal of this section is to compare neural network with dictionary learning by the proposed method in an information-theoretic framework.

Consider a setting with $\bm{X}=[\bm{x}_1, .., \bm{x}_n]$ as the original input signals and  $\bm{Y}=[\bm{y}_1, .., \bm{y}_n]$ as the true label information of the original signals (Fig. \ref{fig:MI}).
Starting with the original input signals, the favorable classification model, in this case, is expected to yield the feature vectors $\bm{A}=[\bm{a}_1, .., \bm{a}_n]$, whose mutual information with the labels $\bm{Y}=[\bm{y}_1, .., \bm{y}_n]$ is maximal.
In addition to seeking the maximum mutual information with the labels, our method additionally seeks to attain maximum mutual information between the feature vectors and the original signals.
To this end, and from an information theoretic point of view, we characterize our classifier as follows,
\begin{equation}
\begin{aligned}
\theta^*=\argmax_{\substack{\theta}}&\; I(\bm{A}^{*(s)}, \bm{Y}),\;\;\;\;\;\;\;\;\;\;\;\;\;\;\;\;(P1)\\
st: \bm{A}^{*(r)}= \argmax_{\substack{\bm{A}^{(r)}}}&\; I(\bm{X}^{*(r-1)}, \bm{A}^{(r)}), r \in \{1, .., s\}
\end{aligned}
\label{eq:MI}
\end{equation}
where $I(\bm{X}, \bm{A})$ is the mutual information between $\bm{X}$ and $\bm{A}$, and $\theta$ is the model parameter. 
The optimization problem in Eqn. (\ref{eq:MI}) seeks $\theta$, by maximizing the mutual information between the labels and the optimal feature vectors, $\bm{A}^*$, which by virtue of the constraint, have maximal mutual information with the original signal.

\begin{figure*}[t]
    \centering
    \begin{subfigure}[b]{0.33\textwidth}
        \centering
        \includegraphics[height=1.6in]{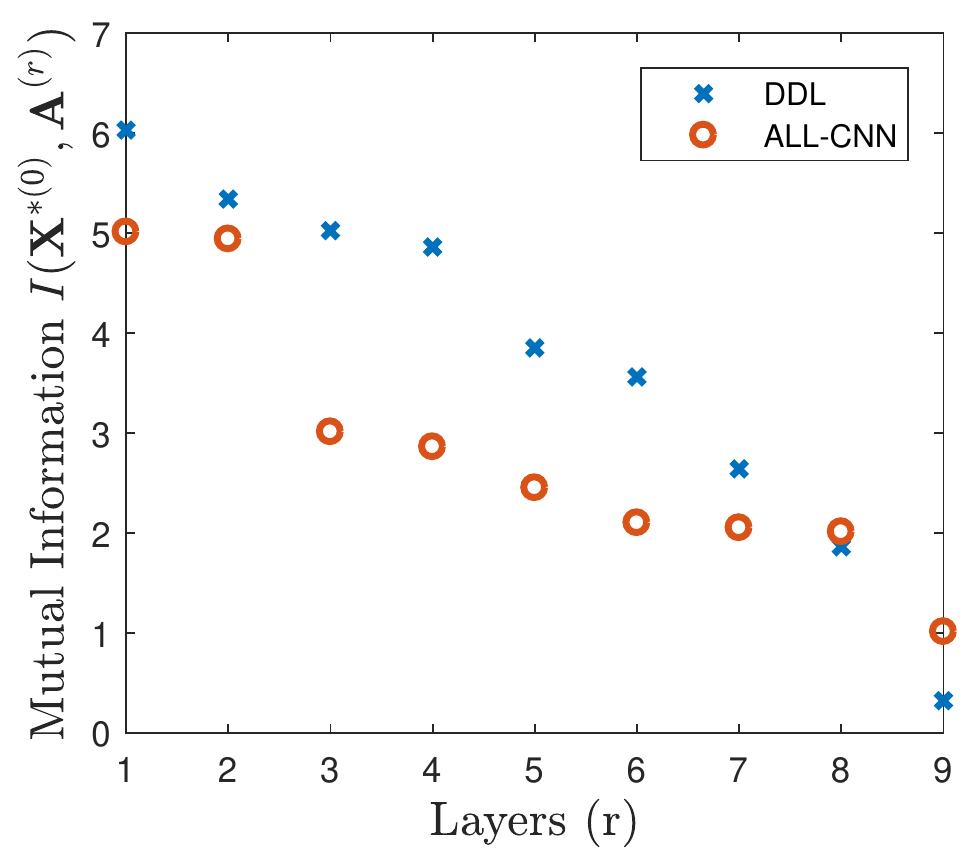}
        \caption{Epoch 1, Iteration 1}
    \end{subfigure}~
    \begin{subfigure}[b]{0.33\textwidth}
        \centering
        \includegraphics[height=1.6in]{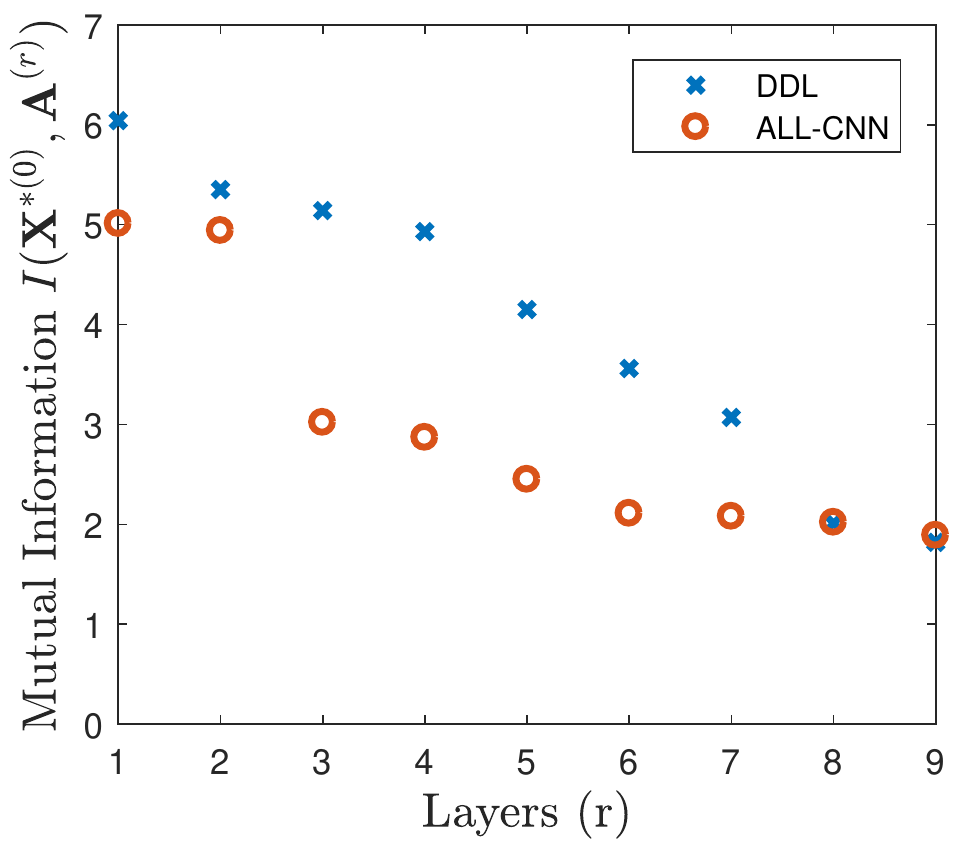}
        \caption{Epoch 1, Iteration 10}
    \end{subfigure}~
    \begin{subfigure}[b]{0.33\textwidth}
        \centering
        \includegraphics[height=1.6in]{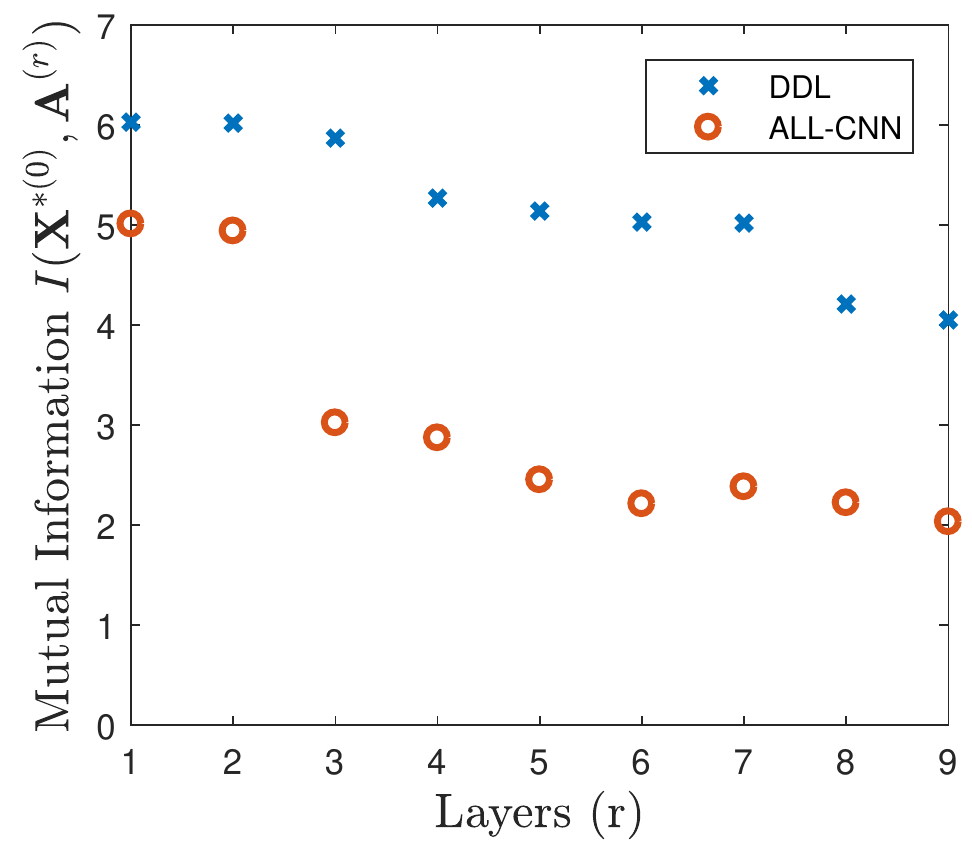}
        \caption{Epoch 25, Iteration 1865}
    \end{subfigure}
    \caption{Mutual informations of output of each layer with the input signal}
    \label{fig:MI_ep}
\end{figure*}
\begin{Proposition}
\textit{Considering the following generative model, we show that the favorable classification model in (\textit{P1}) is approximately equal to a single-layer DL followed by classification by the proposed method:}
\begin{itemize}
  \item \textit{The generated signals} $\{\bm{x}_i\}$ \textit{are independent and follow a Gaussian distribution} $p(\bm{x}_i|\bm{a}_i) \propto e^{-||\bm{x}_i-\bm{D}\bm{a}_i||^2_2}$.
  \item \textit{The features} $\{\bm{a}_i\}$ \textit{are non-negative latent variables with a prior as a compromise between the Gaussian and Laplace priors:} $p(\bm{a}_i) \propto e^{- \lambda ||\bm{a}_i||_1 - \lambda^\prime ||\bm{a}_i||^2_2}$.
\end{itemize}
\end{Proposition}
\begin{proof}
Denoting the Kullback-Leibler divergence measure as $D_{KL}$, $I(\bm{X}, \bm{A})$ can be rewritten as, 
\begin{equation}
\begin{aligned}
I(\bm{X}, \bm{A})=\mathbb{E}_{\bm{x}_i \in \bm{\mathcal{X}}} \Big\{ D_{KL}(p(\bm{a}_i|\bm{x}_i)||p(\bm{a}_i)) \Big\},\\
= \mathbb{E}_{\bm{x}_i \in \bm{\mathcal{X}}} \Big\{ \int_{\bm{a_i}\in \bm{\mathcal{A}}} p(\bm{a}_i|\bm{x}_i) log  \frac{p(\bm{a}_i|\bm{x}_i)}{p(\bm{a}_i)} d\bm{a}_i \Big\},
\end{aligned}
\label{eq:MI2}
\end{equation}
where $\mathbb{E}$ is the expectation operator. Eqn. (\ref{eq:MI2}) can be further simplified by using Bayes' rule as follows,
\begin{equation}
\begin{aligned}
\mathbb{E}_{\bm{x}_i \in \bm{\mathcal{X}}}\Big\{ \int_{\bm{a_i}\in \bm{\mathcal{A}}} \frac{p(\bm{x}_i|\bm{a}_i) p(\bm{a}_i)}{p(\bm{x}_i)} \;\; log \; \frac{p(\bm{x}_i|\bm{a}_i)}{p(\bm{x}_i)} d\bm{a}_i \Big\}.
\end{aligned}
\label{eq:MI3}
\end{equation}
Maximizing Eqn. (\ref{eq:MI3}) with respect to $\bm{A}$ is equivalent to maximizing the following expression,
\begin{equation}
\begin{aligned}
\mathbb{E}_{\bm{x}_i \in \bm{\mathcal{X}}} \Big\{ \int_{\bm{a_i}\in \bm{\mathcal{A}}} p(\bm{x}_i|\bm{a}_i) p(\bm{a}_i)  \;\; log \; p(\bm{x}_i|\bm{a}_i) d\bm{a}_i \Big\}.
\end{aligned}
\label{eq:MI4}
\end{equation}
After some manipulations and maximizing the ``$log$" of the above expression with respect $\bm{A}$, we approximately get the following minimization problem:
\begin{equation}
\begin{aligned}
\argmin_{\{\bm{a}_i\}}\sum_{i=1}^n \Big(||\bm{x}_i - D\bm{a}_i||_2^2 + \lambda||\bm{a}_i||_1 + \lambda^\prime||\bm{a}_i||_2^2 \\ + log ||\bm{x}_i - D\bm{a}_i||_2^2\Big), st: a_i\geq0,
\end{aligned}
\label{eq:MI5}
\end{equation}
which is within an order of $log$ factor different from our formulation in Eqn. (\ref{eq:Hir_dict_learning}). 
\end{proof}

In other words, in contrast to neural networks, preserving a low reconstruction error by minimizing Eqn. (\ref{eq:Hir_dict_learning}) or (\ref{eq:MI5}), leads to higher mutual information between the images and the feature vectors.
We can indeed show that this is due to the smaller conditional entropy of the original signals, $H(\bm{X} | \bm{A})$, in comparison to that achieved by the neural network, $H(\bm{X} | \bm{A}^\prime)$:
\begin{equation}
\begin{aligned}
I(\bm{X},\bm{A})&= H(\bm{X}) - H(\bm{X} | \bm{A}),\\
= H(\bm{X}) - \mathbb{E}_{\bm{a} \in \bm{\mathcal{A}}} \Big\{- &\int_{\bm{x}\in \bm{\mathcal{X}}} p(\bm{x}|\bm{a}) \;\; log \; p(\bm{x}|\bm{a}) d\bm{x} \Big\}.
\end{aligned}
\label{eq:MI6}
\end{equation}
First assume that the conditional probability distribution of the data given the feature vector obtained from our proposed method is $p(\bm{x}|\bm{a}) \propto e^{-||\bm{x}-\bm{D}\bm{a}||_2^2}$ and $p(\bm{x}|\bm{a}^\prime) \propto e^{-||\bm{x}-\varphi^{-1}(\bm{F}^{-1}\bm{a}^\prime)||_2^2}$ for that of the single-layer neural network with a filter $\bm{F}$, and a nonlinear function $\varphi$ such that, $\bm{a}^\prime=\varphi(\bm{F}\bm{x})$. Thus, $H(\bm{X}|\bm{A})$ and $H(\bm{X}|\bm{A}^\prime)$ can be written as,
\begin{equation}
\begin{aligned}
&\footnotesize{H(\bm{X}|\bm{A})=\mathbb{E}_{\bm{a} \in \bm{\mathcal{A}}} \Big\{\int_{\bm{x}\in \bm{\mathcal{X}}} {||\bm{x}-\bm{D}\bm{a}||_2^2} \; e^{-||\bm{x}-\bm{D}\bm{a}||_2^2} \;\; d\bm{x} \Big\},}\\
&\footnotesize{H(\bm{X}|\bm{A}^\prime)}=\\
&\footnotesize{\mathbb{E}_{\bm{a} \in \bm{\mathcal{A}^\prime}} \Big\{\int_{\bm{x}\in \bm{\mathcal{X}}} {||\bm{x}-\varphi^{-1}(\bm{F}^{-1}\bm{a}^\prime)||_2^2} \; e^{-||\bm{x}-\varphi^{-1}(\bm{F}^{-1}\bm{a}^\prime)||_2^2} \;\; d\bm{x} \Big\}}.
\end{aligned}
\label{eq:MI7}
\end{equation}

Neural networks, generally, seek a filter $\bm{F}$ to estimate the ideal classifier $\gamma(\bm{x})$. This may be written as follows,
\begin{equation}
\argmin_{\bm{F}} ||\varphi(\bm{F}\bm{x}) - \gamma(\bm{x})||_2^2.
\label{eq:MI8}
\end{equation}
Hence, the feature vector obtained from a neural network, $\varphi(\bm{F}\bm{x})$, is not optimal to reconstruct the original signal and rather only optimal for classification. 
Our method, however, actively reduces a reconstruction loss (Eqn. (\ref{eq:MI5})) to compute the feature vector.  
Therefore, in Eqn. (\ref{eq:MI7}), the reconstruction loss from a neural network, $||\bm{x}-\varphi^{-1}(\bm{F}^{-1}\bm{a}^\prime)||_2^2$, is larger than $||\bm{x}-\bm{D}\bm{a}||_2^2$. 
This is consistent with experimental results obtained by the authors.

Furthermore, we conducted an experiment to measure the empirical mutual information of the output of each layer with the input signal. 
This experiment compares All-CNN \cite{springenberg2014striving} and a 9-layer DDL over MNIST dataset \cite{lecun1998gradient}.
The mutual information is estimated (from samples \cite{paninski2003estimation}, \cite{kraskov2004estimating}) at different snapshots.
As one can see in Fig. \ref{fig:MI_ep}, our proposed method, over iterations, maintains a higher mutual information between the features and the original signal. This experiment verifies that sequentially conserving the mutual information of input of a layer and output of the layer (Eqn. \ref{eq:MI}), results in an overall high mutual information between the extracted feature vectors and the input signal. 


In summary, this supports the well-known empirical fact about the increasing difficulty of training deeper networks, as the resulting feature vectors become increasingly independent from the original signal (this has also been reported by \cite{shwartz2017opening}). 
Thus, securing the maximum mutual information between the feature vectors and the input signals is important to training deep networks. 
\section{Width of the Network}
\label{sec:Width_of_Network}
A necessary and important step in designing machine learning algorithms is to set/tune the model-hyperparameters. 
In a deep learning framework, the number of filters in each layer is an important model-hyperparameter which in turn, determines the width of each layer, and consequently the fitting capability of the model.
It is always desirable to design a minimal model with a sufficient number of parameters in each layer to learn the true structure of the data. 
In most computer vision problems, the real structure of the data is generally unknown, and picking the number of parameters in each layer is not straightforward. 
An experimentally-proven rule for designing CNN is that, wide layers are good at memorization of the output value for a training dataset, but not so good at generalization in the inference phase. 
While the selection of model-hyperparameters for CNN remains largely heuristic, e.g., random selection search  \cite{bergstra2012random}, or visualization techniques \cite{zeiler2014visualizing}, we provide in this section, some principle in selecting the width of layers in the proposed DDL. Specifically, we propose a rationale for a wider first layer (with a larger number of dictionary atoms) relative to the subsequent layers. 

We show that the second moment of the input signal to each layer is an important factor for selecting the number of dictionary atoms in that layer. 
More specifically, we show that input signals with higher second moments require a larger number of dictionary atoms to keep the reconstruction error small. This consequently, secures a maximal mutual information between the input signal to the layer and the resulting sparse representations. 
Consider a least-squares optimization problem as follows,
\begin{equation}
    \label{eq:WN1}
    \min_{\bm{a}} \frac{1}{2}||\bm{x}-\bm{D}\bm{a}||^2_2 + \lambda ||\bm{a}||_1,
\end{equation}
where $\bm{x}$ is assumed to consist of independent and centered Gaussian entries, with equal variances $\sigma^2$, and the matrix $\bm{D} \in R^{m\times n}$ is a known dictionary. Then, it is desired to characterize the statistical behavior of the optimal solution $\hat{\bm{a}}$ of Eqn. (\ref{eq:WN1}), also called the estimate.
\begin{theorem} 
Considering the least-square optimization problem of Eqn. (\ref{eq:WN1}), the asymptotic value of $E(\hat{\bm{a}}^2)$ is characterized as follows,
\begin{equation}
\label{eq:WN2}
\mathbb{E}[\hat{\bm{a}}^2]= 2(\hat{p}^2+\frac{\lambda^2 \hat{p}^2}{\hat{\beta}^2}) Q(\frac{\lambda}{\hat{\beta}}) - \frac{2 \lambda \hat{p}^2}{\hat{\beta} \sqrt{2 \pi}} exp(-\frac{\lambda^2}{2\hat{\beta}^2}),
\end{equation}
where the function $Q(.)$ is the Gaussian tail Q-function, $\hat{p}$ and $\hat{\beta}$ are the solutions of the following two-dimensional optimization,
\begin{equation}
\label{eq:WN4}
\max_{\beta\geq0}\min_{p>0}\{\frac{p\beta(\gamma-1)}{2}+\frac{\gamma \sigma^2\beta}{2 p} - \frac{\gamma \beta^2}{p} + p F(\beta)\},
\end{equation}
\begin{equation}
\label{eq:WN4}
F(q)=\frac{\lambda e^{-\frac{\lambda^2}{2q^2}}}{2\sqrt{2\pi}}-\frac{q}{2}(1+\frac{\lambda^2}{q^2})Q(\frac{\lambda}{q})+\frac{q}{4},
\end{equation}
and $\gamma=m/n$ determines the ratio of the number of rows to the number of columns (atoms) of the dictionary.
\end{theorem}
\begin{proof} See Appendix A. 
\end{proof}

\begin{figure}[t]
\centering
	\includegraphics[width=2.6in]
	{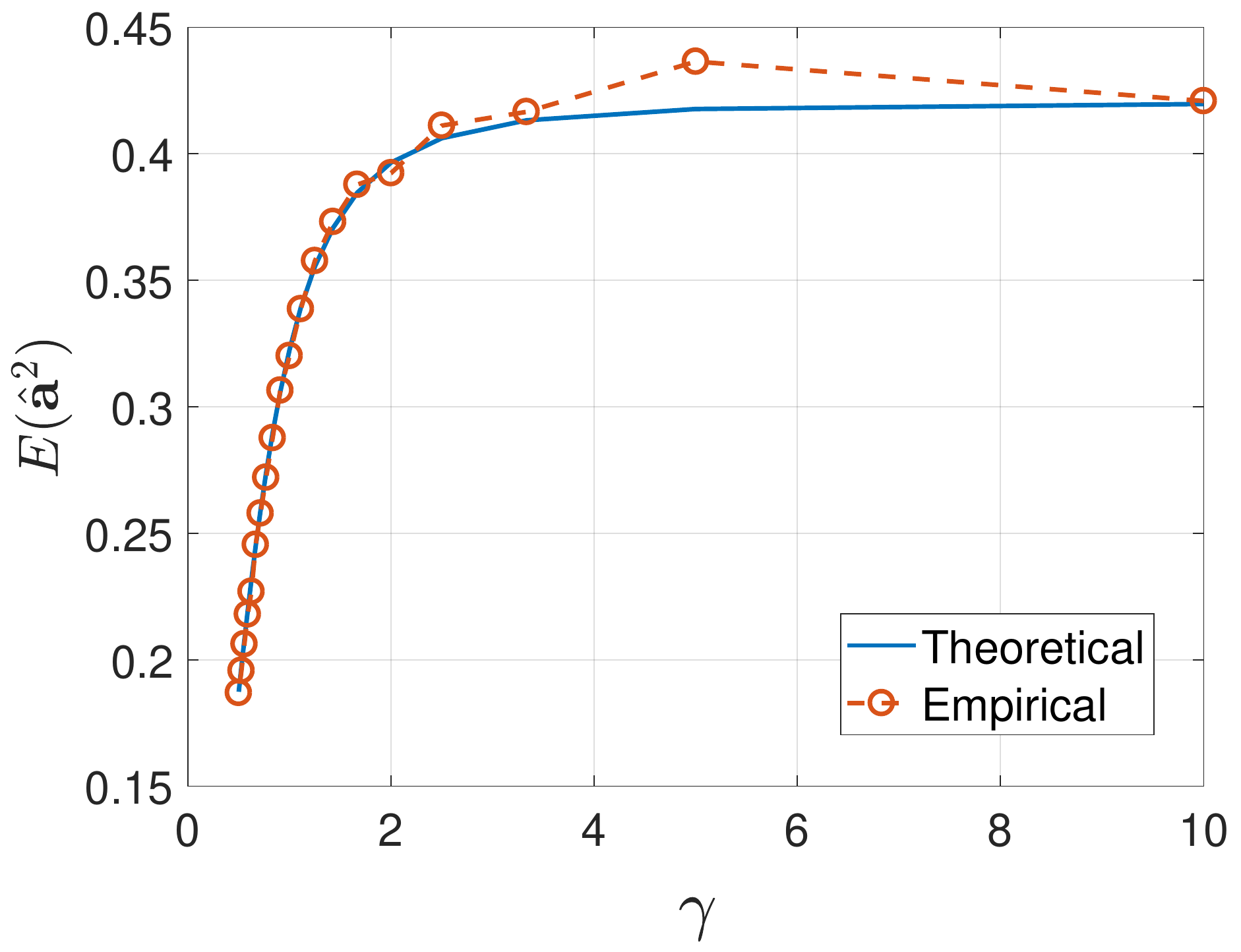}
	\caption{Effect of $\gamma$ on the second moment of the sparse repersentation.}
	\label{fig:NW1}
\end{figure}

Using the above theorem, Fig. \ref{fig:NW1} depicts the value of $\mathbb{E}[\hat{\bm{a}}^2]$  over 50 independent realizations of the LASSO, including
independent Gaussian sensing matrices with $\lambda = 0.5$, and Fig. \ref{fig:NW2} depicts the reconstruction error for different values of $\sigma^2$.
\begin{figure}[t]
\centering
	\includegraphics[width=2.5in]
	{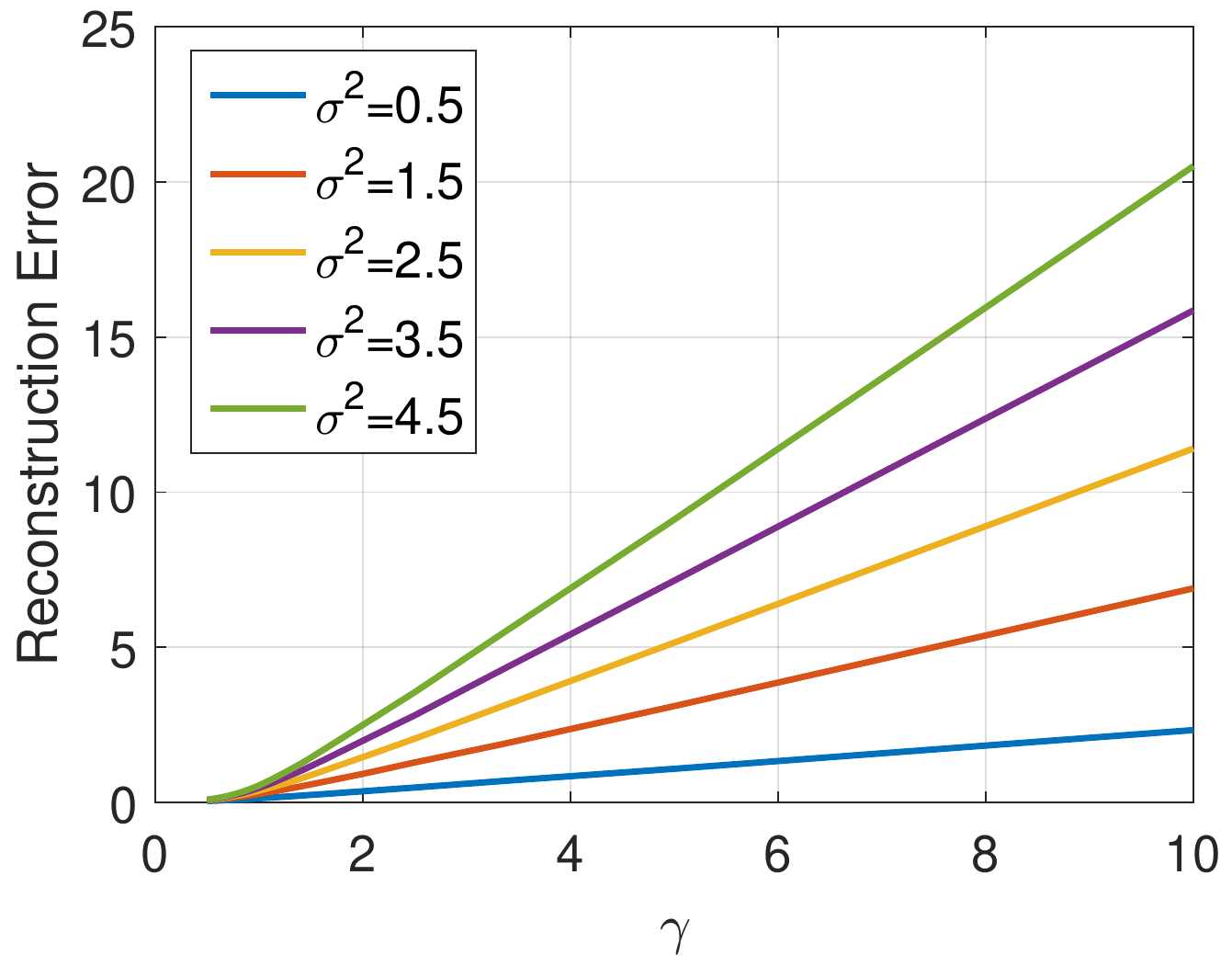}
	\caption{Reconstruction error with varying $\gamma$.}
	\label{fig:NW2}
\end{figure}
As can be seen from Fig. \ref{fig:NW2}, an input with a  higher second moment leads to a higher reconstruction error, and thus amounts to a lower mutual information between the input signal and the sparse code $I(\bm{x}, \bm{a})$. Refer to Eqn. (\ref{eq:MI6}) and Eqn. (\ref{eq:MI7}) for the relation between the reconstruction error and the mutual information.
We can also see from Fig. \ref{fig:NW1}, that decreasing the value of $\gamma$ (increasing the number of dictionary atoms) decreases the second moment of the sparse representation.

In summary, the input signal, $\bm{x}_0$, to the first layer of DDL has a high second moment thus requires a higher number of dictionary atoms (smaller $\gamma$ in Fig. \ref{fig:NW2}) to maintain a low reconstruction error.
Inputs of the next layers of DDL are sparse representations from the output of the previous layers.
With a wide first layer in place, the input to the next layers have smaller second moments (Fig. \ref{fig:NW1}), and a small reconstruction error is assured even when using a smaller number of dictionary atoms.


\section{Experiments}
\label{sec:Results}

Our evaluation of the proposed methodology was carried out on handwritten digits, face images, and object images for recognition tasks using standard datasets. 
The performance of our proposed method is compared to the state-of-the-art methods both in the area of dictionary learning/sparse representation and Convolutional Neural Networks.

Among the dictionary learning/sparse representation methods we compared our model to are the Sparse Representation-based Classification (SRC) \cite{SRC}, Label Consistent K-SVD (LC-KSVD) \cite{jiang2013label}, Discriminative K-SVD (D-KSVD) \cite{zhang2010discriminative}, Locality-constrained Linear Coding (LLC) \cite{LLC}, and task-driven dictionary learning \cite{mairal2012task}.
In addition, we compared our approach to the CNN methods such as invariant Scattering convolution Networks (ScatNet) \cite{bruna2013invariant}, Residual Network (ResNet) \cite{he2016deep}, and Convolutional Kernel Networks (CKN) \cite{mairal2014convolutional}.
Furthermore, to have a fair comparison to CNN methods, we designed deep dictionaries with the same number of parameters and layers as the All-CNN \cite{springenberg2014striving} and performed detailed experiments on both methods.

The classification is based on the features of the last layer of the hierarchy, with a \textit{linear} classifier followed by a ReLU non-linearity ($\varphi(x)=max(0,x)$). 
The classification loss for $n$ images can be written as follows,
\begin{equation}
    \begin{aligned}
    \small{\mathcal{L^C}(\bm{Y}, \bm{X}^{*(s)}, \bm{W})=\frac{1}{n}\big(||\bm{Y}- \varphi(\bm{W} \bm{X}^{*(s)})||_F^2 + \lambda_C||\bm{W}||_F^2\big),}
    \end{aligned}
\end{equation}
where $\bm{Y}$ is a matrix with the class label information of images. The $(j,i)^{th}$ element of this matrix is one if image $i$ belong to Class $j$, and zero otherwise. Columns of $\bm{X}^{*(s)}$ are the inputs to the classifier, and $\bm{W}$ is the classification parameter. We regularize the classification parameter using an $L_2$ norm to prevent overfitting.
To choose the regularization parameters $\lambda$, $\lambda^{\prime}$, and $\lambda_C$, we employed a grid-search on these parameters using 5-fold cross-validation. We increased these parameter values exponentially and selected the set of values which results in the highest performance in the validation set.

\begin{table}[t]
\centering
\caption{Structure of the designed deep dictionaries}
 \begin{tabular}{|c ||c |} 
 \hline
Section 1 & $M_1, 2M_1, 4M_1, M_1$  \\ \hline
Section 2 & $M_2, 2M_2, 4M_2, M_2$ \\ \hline
Section 3 & $M_3, 2M_3, 4M_3, M_3$ \\ \hline
Section 4 & $M_4, 2M_4, 4M_4$ \\ \hline
 \end{tabular}
 \label{tbl:1}
\end{table}
Except for the face recognition experiment, the experiments are carried out on a 15-layer network which can be divided into four sections. 
The overall architecture of the designed deep dictionaries can be found in Table \ref{tbl:1}.
Each row of this table show the number of atoms in a set of layers. For example, assuming $M_1=15$, the first row of the table shows that the first layer dictionary has 15 atoms, and it is followed by dictionaries with 30, 60 and 15 atoms in the next layers. 

Except for the face recognition application, the images are patched into 3 $\times$ 3 atoms with an overlapping stride of 1 in sections one and two, and of stride 2 in sections 3 and 4.
Batch normalization is also used at each layer to prevent very small/very large gradients.

\subsection{Face recognition}
In this regard, we evaluated our proposed algorithm on the Extended YaleB dataset \cite{georghiades2001few}.
This database contains 2,414 face images from 38 individuals. Each individual has about 64 images, and the size
of each image is 192 $\times$ 168 pixels.
Compared to other datasets, images in this dataset are easier to classify. We are therefore, training deep dictionaries with only two layers.
Half of the images per individual are randomly chosen for training, and the other half are used for testing. Due to varying illumination conditions and face expressions, this dataset is a challenging dataset for classification. 

In our approach, the images are partitioned into non-overlapping patches of 24 $\times$ 24 pixels with 200 atoms in the first layer dictionary. Four adjacent patches are concatenated to learn the sparse representations as well as the dictionary in the second layer. The second layer dictionary has 1000 atoms, and the sparse representations have at most 300 non-zero entries. We compare the performance of our method with the state of some art methods in Table \ref{tbl:2}. 
\begin{table}[t]
\centering
\caption{Recognition results on the Extended YaleB dataset}
 \begin{tabular}{|c || c|} 
 \hline
 Method & Accuracy (\%)  \\ [0.5ex] 
 \hline\hline
 SRC \cite{SRC} & 97.2 \\ 
 LLC \cite{LLC} & 90.7 \\
 LC-KSVD \cite{jiang2013label} & 96.7 \\
 \textbf{DDL} & \textbf{99.1} \\ [1ex] 
 \hline
 \end{tabular}
 \label{tbl:2}
\end{table}
Even though the LC-KSVD \cite{LCKSVD} approach is learning discriminative dictionaries via joint classification and dictionary learning, as may be seen from Table \ref{tbl:2}, our approach still registers the highest accuracy. This is due to using the hierarchical approach. On the first layer the elementary details of the images are learned, and the higher level characteristics of the images are learned via the second layer dictionary. 

\subsection{Handwritten digit recognition}
To show the performance of our method on handwritten digit recognition, we use the MNIST dataset \cite{lecun1998gradient}. 
There are totally 70,000 images in this dataset which are divided into 60,000 images for training and 10,000 images for testing. The images in this dataset are of size 28 $\times$ 28.
The number of images in the dataset are very limited, and to avoid overfitting, many state of the art methods learn the features with only a few layers. 
We decided to learn a deep hierarchy of dictionaries to show that the proposed method is not very prone to the overfitting (further  discussed in Subsection \ref{subsec:disc}).

This experiment is carried out on a 15-layer hierarchy of dictionaries with $M_1=10, M_2=20, M_3=30,$ and $M_4=40$.
Table \ref{tbl:3} compares the results of our method with the state of the art methods. As may be seen from this table, our approach registers the highest accuracy. The first group of algorithms, SRC \cite{SRC}, D-KSVD \cite{zhang2010discriminative}, LC-KSVD \cite{jiang2013label}, and task-driven dictionary learning \cite{mairal2012task}, learn the sparse features from the entire image. While, the second group, invariant scattering convolution \cite{bruna2013invariant}, Convolutional Kernel Networks \cite{mairal2014convolutional}, and our proposed method, learn the feature vectors over multiple scales.
Comparing the performance of the first group of works and the second group shows the necessity of learning sparse features over multiple scales. 
\begin{table}[t]
\centering
\caption{Recognition results on the MNIST dataset}
 \begin{tabular}{|c || c|} 
 \hline
 Method & Error percentage (\%)  \\ [0.5ex] 
 \hline\hline
 SRC \cite{SRC} & 4.31 \\ 
 D-KSVD \cite{zhang2010discriminative} & 9.67 \\
 LC-KSVD \cite{jiang2013label} & 7.42 \\
 Task-driven \cite{mairal2012task} & 0.54\\
 Invariant scattering \cite{bruna2013invariant} & 0.43\\
 CKN \cite{mairal2014convolutional} & 0.39\\
 \textbf{DDL} & \textbf{0.32} \\ [1ex] 
 \hline
 \end{tabular}
 \label{tbl:3}
\end{table}

\subsection{Object classification}
For an object classification task, we evaluated our method on two challenging datasets of CIFAR-10 and CIFAR-100.
There are 60,000 color images in CIFAR-10 dataset which are divided into 50,000 training images and 10,000 test images. The size of images in this dataset is 32$\times$32 for a total of in 10 classes.
The images in CIFAR-100 are also 32$\times$32. However, this dataset has 100 classes with 500 training images and 100 testing images per class. 
This experiment is carried out on a 15-layer hierarchy of dictionaries with $M_1=15, M_2=30, M_3=45,$ and $M_4=60$.

Fig. \ref{fig:cifar_acc} shows the learning curve of the proposed  method over the CIFAR-10 dataset, and Table \ref{tbl:4} compares the accuracy of the proposed method with other state of the art methods which have all approximately comparable number of parameters.
\begin{figure}[t]
\centering
\includegraphics[scale=0.7]
{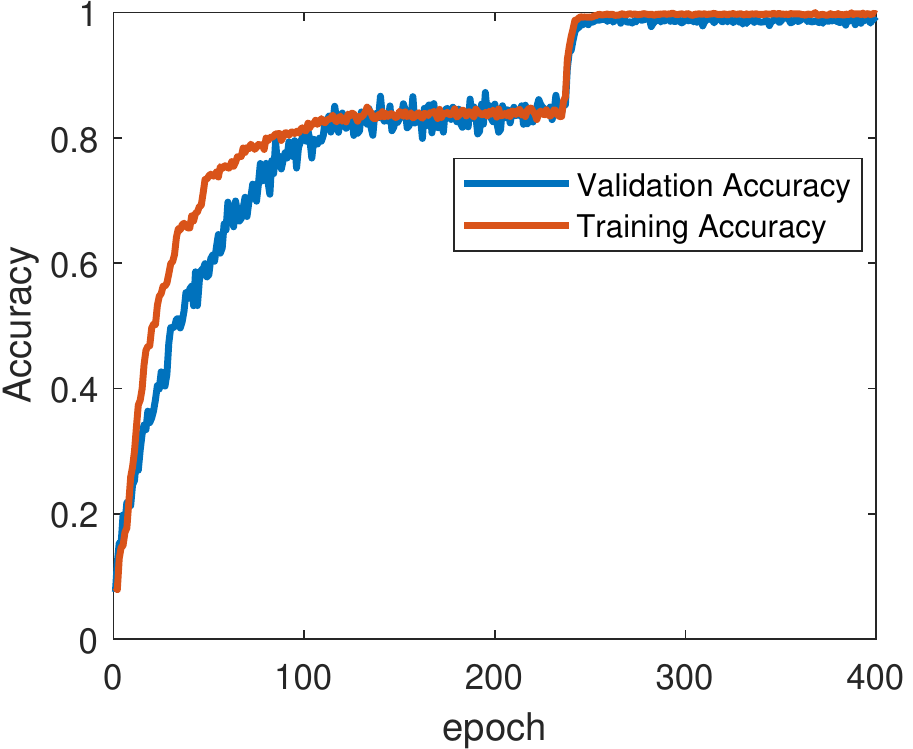}
\caption{Changes in the classification accuracy of the proposed method on the CIFAR-10 dataset over different epochs.}
\label{fig:cifar_acc}
\end{figure}
\begin{table}[t]
\centering
\caption{Recognition accuracy (in percentage) on the CIFAR-10 and CIFAR-100 datasets}
 \begin{tabular}{|c || c|| c|| c |} 
 \hline
 \footnotesize{Method} & \footnotesize{\#Params} & \footnotesize{CIFAR-10} & \footnotesize{CIFAR-100}  \\ [0.1ex] 
 \hline\hline
 \small{All-CNN \cite{springenberg2014striving}} & \small{$\approx$1.4M} & 92.75 & 66.29\\
 \small{CKN \cite{mairal2014convolutional}} & \small{$\approx$0.32M} & 78.30 & - \\
 \small{ResNet \cite{he2016deep}} & \small{$\approx$0.85M} & 93.59 & 72.78\\
 \small{\textbf{DDL 9-layers}} & \small{$\approx$1.4M} & \textbf{93.04} & \textbf{68.76}  \\
 \small{\textbf{DDL 15-layers}} & \small{$\approx$0.76M} & \textbf{94.17} & \textbf{80.62}  \\ [0.1ex] 
 \hline
 \end{tabular}
 \label{tbl:4}
\end{table}
As may be seen from the table, our proposed method obtains a higher accuracy in both datasets with a smaller number of training parameters. 
\begin{figure}[b]
    \centering
    \begin{subfigure}[b]{0.5\textwidth}
        \centering
        \includegraphics[height=0.9in]{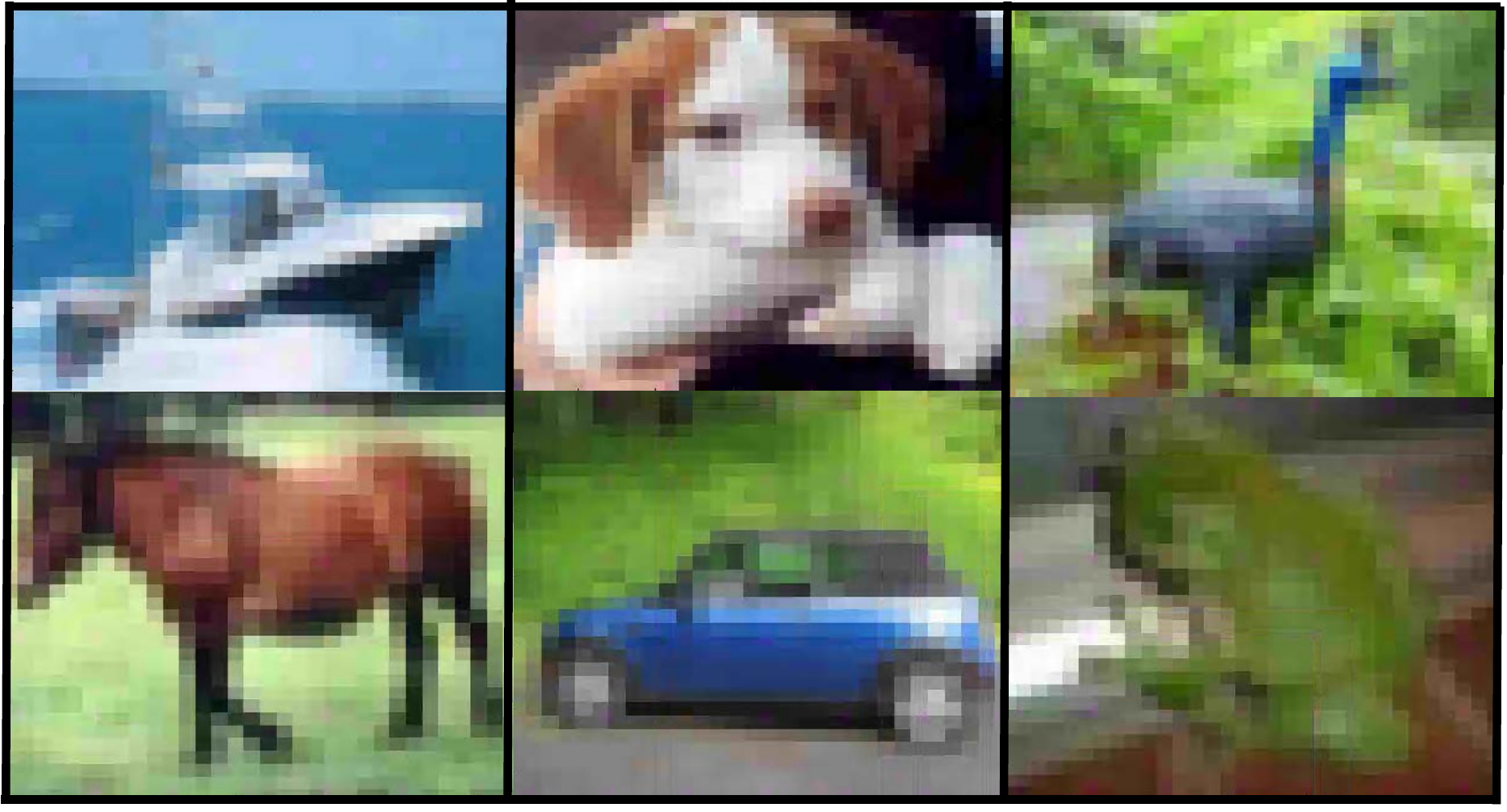}
        \caption{Original images}
    \end{subfigure}
    \begin{subfigure}[b]{0.5\textwidth}
        \centering
        \includegraphics[height=0.9in]{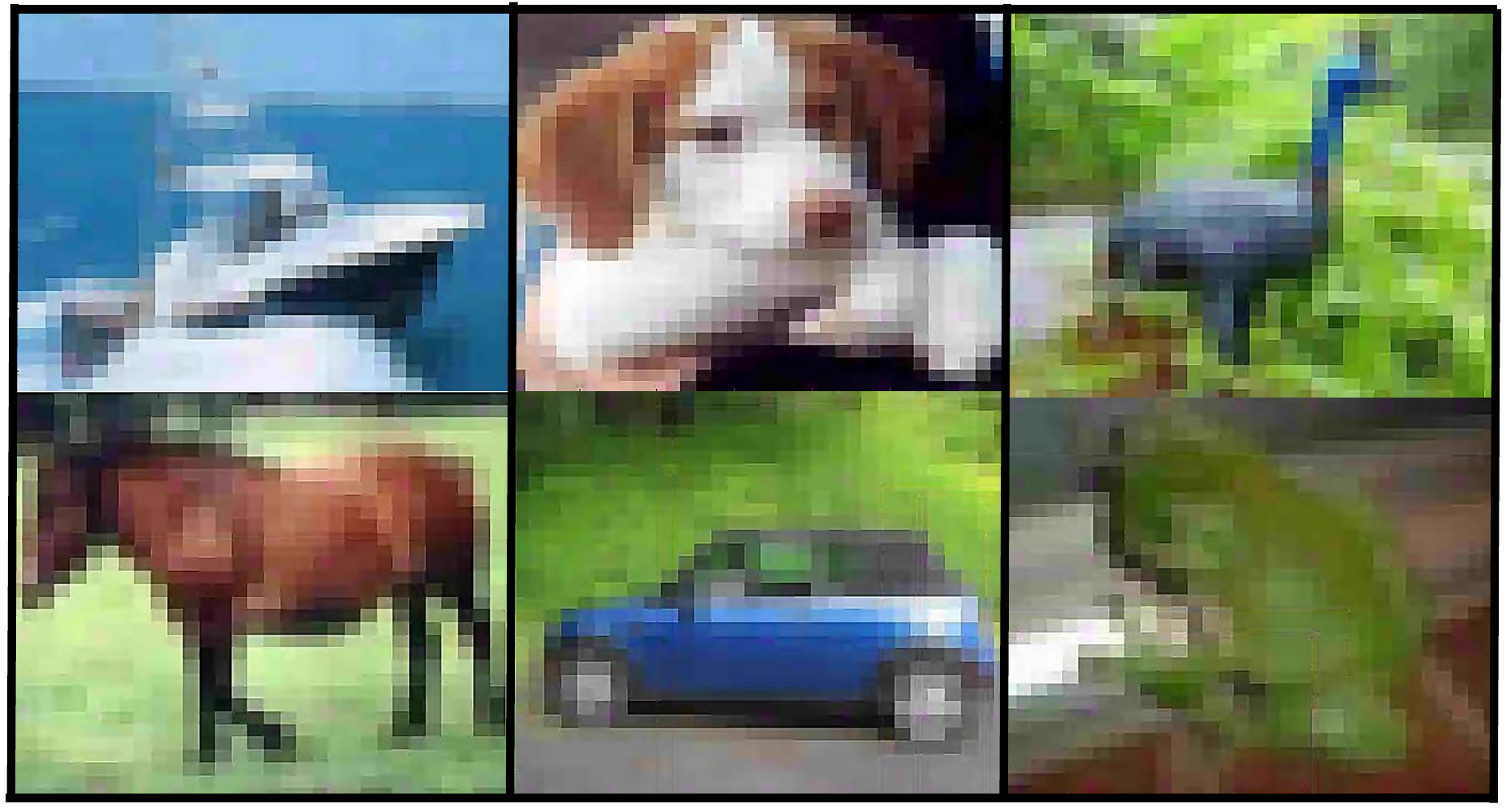}
        \caption{Reconstructed images}
    \end{subfigure}
    \caption{comparison between the reconstructed images form DDL and the original images form CIFAR-10 dataset}
    \label{fig:recon}
\end{figure}
To have a fair comparison with classical CNN, and be independent of the number of layers and parameters, the $4^{th}$ row of the table shows the accuracy of the proposed method in a 9-layer design (and a similar architecture to the All-CNN  \cite{springenberg2014striving}). As may be seen, with the same number of parameters, the deep dictionary achieves a higher accuracy. The lower accuracy of the All-CNN is primarily due to the difficulty in tuning deeper networks rather than their intrinsic capacity. Fig. \ref{fig:recon} displays a subset of reconstructed images via DDL and their corresponding original images  in CIFAR-10. The reconstructed images are obtained by sequentially reshaping and multiplying the feature vectors to each layer's dictionary. As may be seen from this figure, the reconstructed images are very similar to the original images, which is due to the effect of minimizing the reconstruction loss while constructing the feature vectors and preserving maximum mutual information with the original images.

\subsection{Robustness to adversarial perturbations}
\label{subsec:noise}
Many researchers have recently reported the vulnerability of state-of-the-art deep learning techniques to adversarial perturbations \cite{Dezfooli2017, deepfool, goodfellow2014explaining}. 
These studies have shown that one can find a single small additive image perturbation to fool deep learning algorithms. The additive adversarial perturbation $\bm{v}$ (based on the definition in \cite{Dezfooli2017, deepfool}) should satisfy the following constraint,
$||\bm{v}||_2 ~ \leq ~\rho ~ \mathbb{E}_{\bm{x}} ||\bm{x}||_2$.
This constraint controls the magnitude of the additive noise, where $\rho$ is usually a small number, and $\mathbb{E}_{\bm{x}}$ is the expected value of the magnitude of images. Upon inducing noise, the fooling rate can be calculated as, $\mathbb{P}_{\bm{x}}(\hat{\theta}(\bm{x}+\bm{v})\neq\hat{\theta}(\bm{x}))$, where $\hat{\theta}(\bm{x})$ is the estimated label for image $\bm{x}$.

To investigate the robustness of our algorithm, we used the algorithm in \cite{Dezfooli2017} with $\rho=0.04$ to compare the fooling rate of our proposed algorithm in presence of an adversarial perturbation (Table \ref{tbl:5}). Our proposed algorithm displays great resilience with a much lower fooling rate in comparison to the other algorithms (DenseNet \cite{resnet}, VGG-19 \cite{vgg}, and RES Net \cite{resnet}).
\begin{table}[t]
\centering
\caption{Fooling rate (in percentage) on the CIFAR-10 dataset}
 \begin{tabular}{|c || c|} 
 \hline
 \footnotesize{Classifier} & \footnotesize{Fooling rate} \\ [0.1ex] 
 \hline\hline
 VGG-19 \cite{vgg} & 0.67\\
 RES Net-101 \cite{resnet} & 0.84\\
 Dense Net  \cite{densenet} & 0.77\\
 \textbf{DDL}  & \textbf{0.09}\\
[0.1ex] 
 \hline
 \end{tabular}
 \label{tbl:5}
\end{table}
We further study the robustness of our algorithm, by adding a single random noise to the images. Figs. \ref{fig:fooling_mnist} and \ref{fig:fooling_cifar} compare the robustness of our algorithm to the state of the art deep learning classifiers in presence of random additive noise in MNIST and CIFAR-10 datasets respectively.
\begin{figure}[t]
\centering
\includegraphics[scale=0.5]{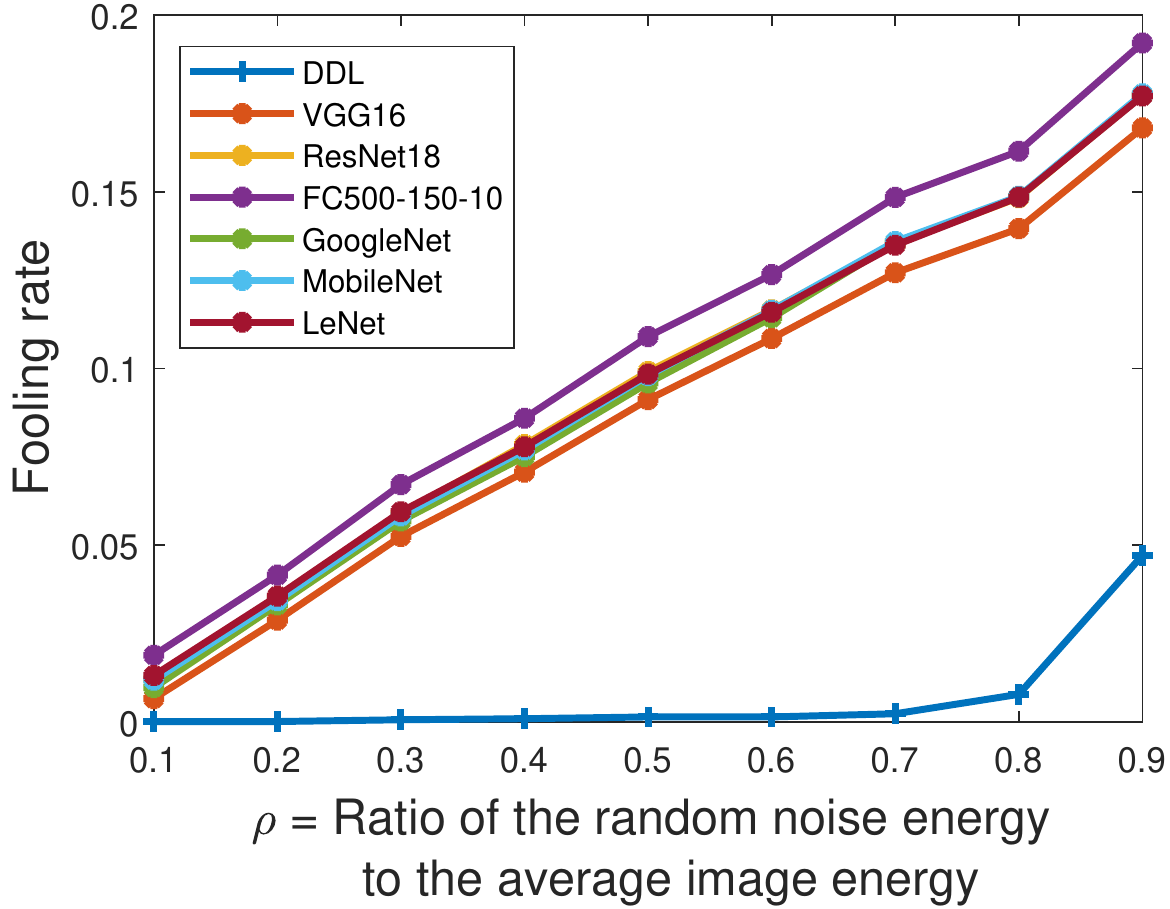}
\caption{Fooling rate in MNIST dataset}
\label{fig:fooling_mnist}
\end{figure}
\begin{figure}[t]
\centering
\includegraphics[scale=0.5]{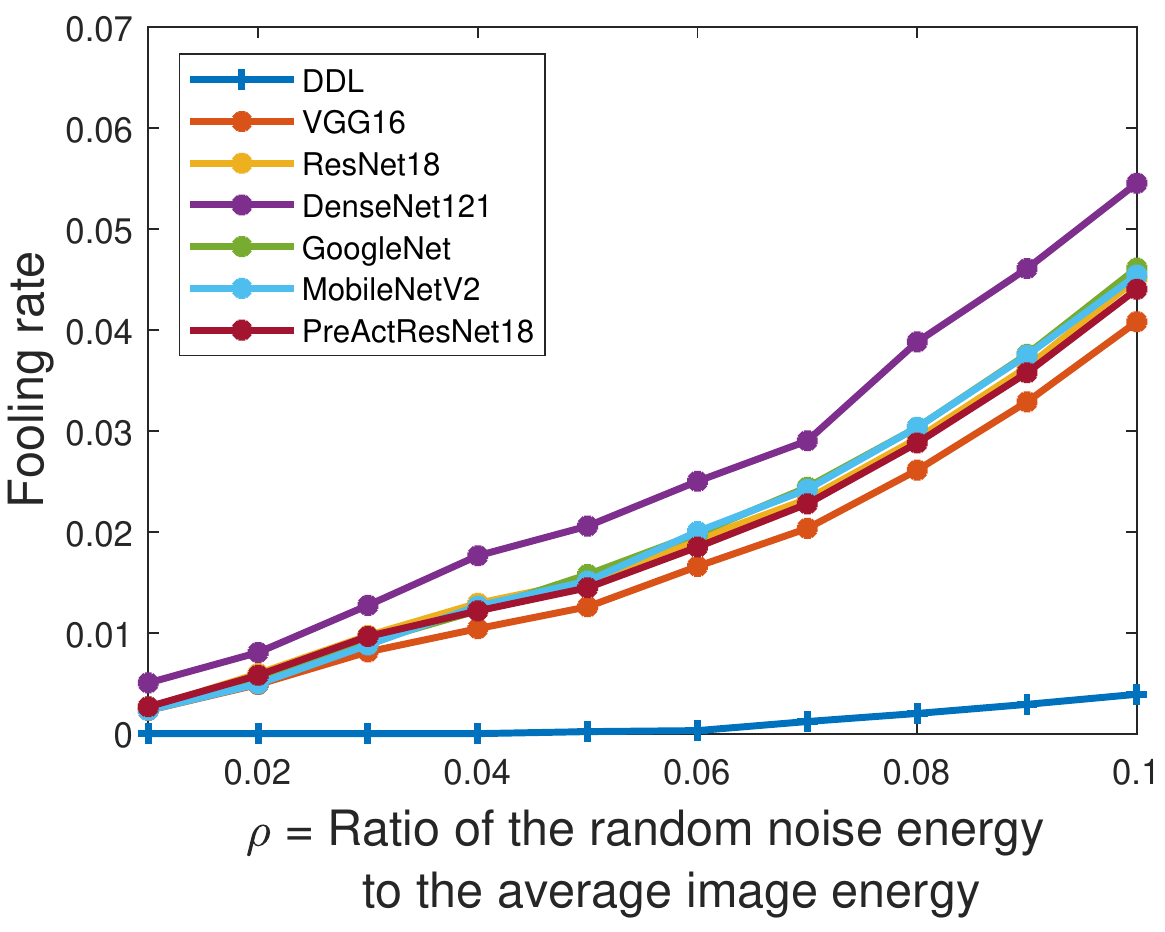}
\caption{Fooling rate in CIFAR-10 dataset}
\label{fig:fooling_cifar}
\end{figure}
In Convolutional Neural Networks the additive noise propagates through the layers. 
In contrast to CNN, the image features in our proposed algorithm are represented by parsimoniously selecting a minimum number of basis vectors (dictionary atoms). This improves the robustness of our proposed approach and its tolerance to adversarial perturbations.

\subsection{Generalizability to deeper networks}
\label{subsec:disc}
\begin{figure}[t]
\centering
\includegraphics[scale=0.7]
{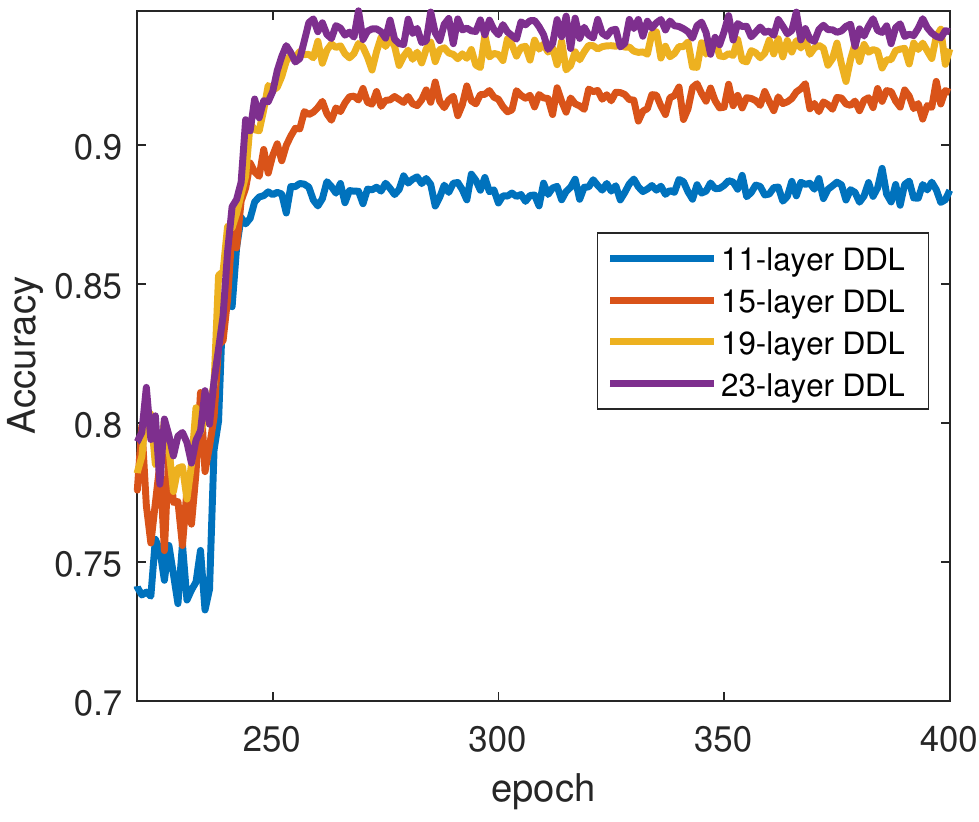}
\caption{Learning curves of the proposed method with different number of layers on the CIFAR-10 dataset.}
\label{fig:cifar_layers}
\end{figure}

In this section, we trained 4 different networks with a different number of layers to test their performance on the CIFAR-10 dataset. The design of the 15-layer network is shown in Table. \ref{tbl:1}. The 11-layer network is similar to the 15-layer network but trimmed at the 11$^{th}$ layer. The 19-layer network is built by putting the $M_4, M_4, 2M_4$, and $4M_4$ layers at the end of the 15-layer network, and the 23-layer is similarly, built by repeating the $M_4, M_4, 2M_4$, and $4M_4$ layers at the end of the 19-layer network. 
As Fig. \ref{fig:cifar_layers} shows, the accuracy of our proposed approach increases as the structure depth increases. 
By adding more layers, the extracted feature vectors are maintaining high mutual information with the original signal thereby facilitating the training of the deep dictionaries in contrast to the Convolutional Neural Networks, and making this method a better candidate for more elaborate learning tasks such as transfer learning. 
\section{Conclusions} 
\label{sec:conclusions}
In this paper, we used four image datasets to evaluate the classification performance of our proposed deep dictionary learning method. 
We demonstrated the importance of representing images by learning image characteristics at multiple scales via deep dictionaries.
The importance of preserving the maximum mutual information between the feature vectors and the input signals is discussed from an information theoretic perspective, and is tested empirically.
We also showed that refining the dictionary learning and feature selection by accounting for the target task improves performance.
The evaluation results show the merit of the proposed method for classifying  images.
\appendices
\section{Proof of Theorem 1:}
Consider the following regularized least-square optimization problem,
\begin{equation}
    \label{eq_pr1}
    \min_{\bm{a}} \frac{1}{2} ||\bm{z}-\bm{D}\bm{a}||_2^2 + f(\bm{a}),
\end{equation}
where matrix $\bm{D}\in R^{m\times n}$ is the dictionary matrix, $\bm{z}\in R^m$ and $\bm{a}\in R^n$ respectively are the signal and the coefficient in the dictionary of atoms, and $f(\bm{a})$ is a real and convex function.
Understanding the asymptotic behaviour of the solution of the regularized least squares problem where $m$ and $n$ grow to infinity with a constant ratio $\gamma=m/n$ is an interesting case.
A scenario which is widely considered in the literature is when $\bm{z}$ is generated by a linear model as follows,
\begin{equation}
    \label{eq_pr2}
    \bm{z}=\bm{D}\bm{a}_0 + \bm{v},
\end{equation}
where $\bm{a}_0$ is the true structured vector and $\bm{v}$ is a noise vector with centered Gaussian entries, with equal variance $\sigma^2$.
Using Gordon's min-max Theorem (\cite{Gordon88}, Lemma 3.1 ) the authors in \cite{hasibi2015} showed that the empirical distribution of $\hat{\bm{a}}$ converges to that of $\hat{\bm{A}}$,
\begin{equation}
    \label{eq_pr3}
    \hat{\bm{A}}=\argmin_{\bm{a}} \frac{\hat{\beta}}{2\hat{p}}(\bm{a}-\bm{a}_0+\hat{p}\Gamma)^2 + f(\bm{a}),
\end{equation}
where $\Gamma$ is a standard Gaussian vector and ($\hat{\beta}$, $\hat{p}$) are the results of the following optimization,
\begin{equation}
    \label{eq_pr4}
     arg \max_{\beta\geq0} \min_{p>0} \{\frac{p\beta(\gamma-1)}{2} + \frac{\gamma\sigma^2 \beta}{2p} - \frac{\gamma\beta^2}{2} + \mathbb{E}[(S_f(\frac{\beta}{p},p\Gamma+A))] \}.
\end{equation}
Further, $S_f(.,.)$ denotes the proximal function of $f$ which is defined as,
\begin{equation}
    \label{eq_pr5}
     S_f(q,y)= \min_x \frac{q}{2} (x-y)^2 +f(x).
\end{equation}
For $f(a)=\lambda||a||_1$, the solution to Eqn. (\ref{eq_pr4}) was derived for LASSO and is given in Eqn. (\ref{eq_pr6}) \cite{panahi2017universal}.
\begin{equation}
    \label{eq_pr6}
     \mathbb{E}[(S_f(\frac{\beta}{p},p\Gamma+A))]=k\sqrt[]{1+p^2}F(\frac{\beta}{p}\sqrt[]{1+p^2})+(1-k)pF(\beta),
\end{equation}
where 
\begin{equation}
    \label{eq_pr7}
     F(q)=\frac{\lambda e^{-\frac{\lambda^2}{2q^2}}}{2\sqrt[]{2\pi}}-\frac{q}{2}(1+\frac{\lambda^2}{q^2}) Q(\frac{\lambda}{q}) + \frac{q}{4}.
\end{equation}
The entries of the true vector $\bm{a}_0$ are assumed to be non-zero and standard Gaussian with probability $0 \leq k \leq 1$. In other words, the distribution of a $\bm{a}_0$ is $\xi = k \mathcal{N}+ (1-k)\sigma_0$, where $\mathcal{N}$ and $\sigma_0$ are standard Gaussian and the Dirac measures on R, respectively.

The function $Q(.)$ is a Gaussian tail Q-function. Replacing the above expression in Eqn. (\ref{eq_pr4}), we can obtain $\hat{p}$ and $\hat{\beta}$. The asymptotic value of $\mathbb{E}[(\hat{\bm{a}}-\bm{a}_0)^2]$ can thus be obtained as following term:
\begin{equation}
    \label{eq_pr8}
     \mathbb{E}[(\hat{\bm{a}}-\bm{a}_0)^2]=kJ(\frac{\lambda p}{\beta},p,1)+(1-k) J(\frac{\lambda p}{\beta},p,0),
\end{equation}
where
\begin{equation}\label{eq_pr9}
\begin{split}
     J(\epsilon,p,\alpha)=&\alpha^2 +2(p^2 + \epsilon^2 - \alpha^2) Q(\frac{\epsilon}{\sqrt[]{\alpha^2+p^2}})-\\
      &- 2\epsilon \; \sqrt[]{\frac{\alpha^2+p^2}{2\pi}} exp(-\frac{\epsilon^2}{2(\alpha^2+p^2)}).
\end{split}
\end{equation}

Upon setting $k=0$ in the above equation and some simplifications, we can finally calculate $\mathbb{E}[\hat{\bm{a}}^2]$ as Eqn. (\ref{eq:WN2}).

\ifCLASSOPTIONcompsoc
  \section*{Acknowledgments}
\else
  \section*{Acknowledgment}
\fi
We would like to acknowledge the support of U.S. Army Research Office: Grant  W911NF-16-2-0005 and we thank Dr. Michael Bronstein for suggesting to consider a comparative evaluation under adversarial perturbation.

\ifCLASSOPTIONcaptionsoff
  \newpage
\fi



%

\bibliographystyle{ieee.bst}
\bibliography{main.bib}




%








\end{document}